\documentclass[journal]{IEEEtran}
\usepackage{amsmath,amssymb,amsfonts,amsthm}
\usepackage{algorithmic}
\usepackage{algorithm}
\usepackage{graphicx}
\usepackage{textcomp}
\usepackage{xcolor}
\usepackage{cite}
\usepackage{bm}
\usepackage{booktabs}
\usepackage{multirow}
\definecolor{revblue}{rgb}{0,0,0.72}
\newif\ifRevHL \RevHLfalse
\newcommand{\revB}{}
\newcommand{\revE}{}
\protected\long\def\rev#1{{#1}}
\protected\long\def\revb#1{\ifRevHL{\color{revblue}{}#1}\else{#1}\fi}
\newtheorem{proposition}{Proposition}
\newtheorem{corollary}{Corollary}

\newcommand{\vx}{\mathbf{x}}
\newcommand{\vy}{\mathbf{y}}
\newcommand{\vmu}{\boldsymbol{\mu}}
\newcommand{\mQ}{\mathbf{Q}}
\newcommand{\mSig}{\boldsymbol{\Sigma}}
\newcommand{\mG}{\mathbf{G}}
\newcommand{\E}{\mathbb{E}}
\newcommand{\Var}{\mathrm{Var}}
\DeclareMathOperator*{\argmax}{arg\,max}

\begin{document}
\bstctlcite{IEEEexample:BSTcontrol}

\title{INLA for Adaptive UAV Sampling of Radio Environment Maps in IoT Networks:\\
\rev{Prior Width Governs Accuracy and Calibration}}

\author{
\IEEEauthorblockN{Changhao He, Kaichuang Yang, Xiaojuan Zhang,~\IEEEmembership{Senior Member,~IEEE}}
\thanks{C. H. He, K. C. Yang are with King Abdullah University of Science and Technology, Saudi Arabia, (\{changhao.he, kaichuang.yang\}@kaust.edu.sa). X. J. Zhang is with the Institute of Advanced Intelligence and Computing, Agency for Science, Technology and Research, Singapore, (xiaojuanzhang@ieee.org).}}

\maketitle

\begin{abstract}
We study adaptive UAV sampling for radio environment map (REM) construction
from non-Gaussian IoT observations. Integrated nested Laplace approximation
(INLA) updates a Mat\'ern GMRF posterior, with waypoints selected by entropy
pre-screening and Sherman--Morrison lookahead over the reachable set.
We show that the likelihood-induced Skew-Normal correction has a sextic
entropy effect, explaining its limited influence on waypoint selection.
Across 19 retained ray-traced windows,
\rev{Random and lawnmower sampling cost a median 18.0\% and 14.0\% more residual NMSE
than posterior-informed sampling, equivalently reductions of 15.3\% and 12.3\% when referred to
those baselines}, while differences among informative rules remain unresolved.
Widening the prior shadowing standard deviation from 6.0 to 10.025\,dB
improves NMSE and credible-interval coverage,
\revb{with monotone improvement over a 4--16\,dB sweep on this corpus and
no observed turning point,} although overconfidence persists.
\revb{Online width estimation from the UAV's own counts yields about
6.5\,dB without recovering the gain. Across 30 measured LoRaWAN propagation
windows paired with simulated counts and scored only on received cells,
the margin over non-adaptive sampling reproduces, the training-derived
width minimizes the tested interval score, and GP-UCB has nominally
$3$--$5\%$ lower NMSE.}
\revb{The results show substantial, corpus-dependent sensitivity to
prior specification, while fine differences among acquisition rules
have limited or unresolved effects, the preferred width depends on
the corpus, scored cells and model mismatch.}
\end{abstract}

\begin{IEEEkeywords}
Adaptive sampling, Bayesian experimental design, Gaussian Markov random field, informative
path planning, INLA, Internet of Things (IoT),
posterior calibration, prior sensitivity, radio environment map (REM), UAV.
\end{IEEEkeywords}

\section{Introduction}
\label{sec:intro}

\rev{An operator deploying a $500{\times}500$\,m industrial Internet-of-Things (IoT)
site must decide where to add gateways, which cells require further measurements,
and when to stop surveying. These decisions rely on a radio environment map
(REM)~\cite{Yilmaz2013}, which an unmanned aerial vehicle (UAV) can construct.
With $10^4$ candidate locations at $5$\,m resolution and flight endurance of
tens of minutes, exhaustive measurement is infeasible~\cite{Sun2025fm}. Random and lawnmower sampling ignore previous measurements. Model-based
strategies instead select locations using a Gaussian process or Kriging
posterior~\cite{Shrestha2023,Polyzos2024}, commonly by maximizing predictive
variance.

With this model,} local single-step criteria can induce identical rankings.
A hypothetical observation at location~$j$ reduces local uncertainty according
to $c_j\sigma_j^2$, where $c_j = -\ell''(\mu_j)$ is the likelihood curvature
at the posterior mean. With additive Gaussian noise, $c_j \equiv c$ is constant,
so the one-step entropy reduction $\tfrac{1}{2}\log(1+c\sigma_j^2)$ is strictly
monotone in $\sigma_j^2$. Entropy, variance and local one-step mutual information
therefore induce identical rankings. We call this the \emph{Gaussian equivalence
ceiling}\rev{: changing between these local criteria does not change the selected
cell. This statement concerns local scores, not covariance-based criteria such
as the summed-marginal lookahead of~\eqref{eq:DeltaU-cancel}.}

Non-Gaussian IoT observations remove the constant-curvature condition.
Binary spectrum-sensing decisions~\cite{Haykin2005} give Bernoulli likelihoods
with $c_j = \beta_1^2 p_j(1{-}p_j)$, while random-access preamble
counts~\cite{Baccelli2009} give Poisson likelihoods with
$c_j = N_e e^{\mu_j}$, which can vary across the field by orders of magnitude.
The posterior can then be skewed, and the integrated nested Laplace
approximation (INLA)~\cite{Rue2009} estimates this skewness during inference.
We investigate whether it improves waypoint selection and what governs
REM accuracy and posterior calibration.

\rev{We study acquisition and prior specification within a common framework.}
\revb{INLA with its simplified-Laplace correction~\cite{Rue2009}, the
Mat\'{e}rn--SPDE prior~\cite{Lindgren2011}, the Skew-Normal
family~\cite{Azzalini1985} and sequential Bayesian experimental design
are inherited components. Our contributions are the skewness-aware
acquisition rule, its analysis and the following empirical findings:}
\begin{enumerate}
\item We formulate adaptive REM construction from Poisson and Bernoulli
IoT observations using an INLA posterior and a kinematically constrained
UAV sampler, with an entropy pre-screen followed by a covariance lookahead.

\item We quantify the effect of prior width on reconstruction accuracy
and posterior calibration. Changing only the prior shadowing standard
deviation from 6.0 to 10.025\,dB improves coverage by 7.3 percentage points
and the six-rule pooled median NMSE by 8.9\%.
\revb{The improvement is monotone over $4$--$16$\,dB on the ray-traced
corpus, whereas the training-derived residual scale minimizes the tested
interval score on received cells of the measured corpus. Online marginal-likelihood estimation settles
near $6.5$\,dB on the ray-traced corpus and does not recover the
better-performing wider prior, even with a learned window-level mean
offset.}

\item We characterize the skewness correction theoretically and
experimentally. The Skew-Normal entropy gap is
$\Delta H(\alpha)=C_6\alpha^6+O(\alpha^8)$, and its isolated accuracy
effect is statistically unresolved over the tested regime.

\item We evaluate the framework on 19 held-out ray-traced maps
\revb{and the 21 excluded maps, and on 30 windows of measured LoRaWAN
propagation fields paired with simulated counts, scored on received cells.}

\item \revb{We assess sensitivity to prior-width misspecification
(synthetic corpus), observation-model misspecification (ray-traced corpus)
and non-stationary traffic (synthetic corpus), and report flight distance,
computational cost and a normalized link-budget diagnostic on the
synthetic and ray-traced corpora.}
\end{enumerate}

\section{\rev{Related Work}}
\label{sec:related}

\noindent\textbf{REM construction and radio-map estimation.}
REMs support spectrum management~\cite{Yilmaz2013}, with sparse reconstruction
surveyed as spectrum cartography~\cite{Romero2022}. Approaches include
low-rank estimators and image-based map completion~\cite{Levie2021}.
Many formulations use given measurement locations and additive-Gaussian
received-power observations. Here, locations are selected sequentially under
UAV motion constraints, retaining counting and detection likelihoods.

\noindent\textbf{UAV informative path planning.}
Methods include continuous-space Bayesian optimization under a GP
posterior~\cite{Marchant2014} and learned planning
policies~\cite{Bayerlein2021}. Wireless applications include online UAV
spectrum surveying~\cite{Shrestha2023} and planning with known radio maps
and rotary-wing propulsion costs~\cite{ZengEnergy2019}. Our study examines
non-Gaussian acquisition and evaluates posterior calibration alongside
reconstruction accuracy.

\noindent\textbf{Bayesian active learning and INLA.}
GP sensor placement distinguishes mutual-information objectives from
marginal entropy and variance~\cite{Krause2008}, GP-UCB provides
exploration bonuses with regret guarantees~\cite{Srinivas2012}, and
GP active learning supports 5G radio-map
reconstruction~\cite{Polyzos2024}. INLA provides deterministic
inference~\cite{Rue2009}, the Mat\'{e}rn--SPDE representation supplies
a sparse prior~\cite{Lindgren2011}, and recent work develops scalable
skewed posterior approximations~\cite{Dutta2026}. Related methods use
Laplace-approximated Poisson posteriors for Bayesian
optimization~\cite{Mei2024,yang2026isbo}. Here, we instead use posterior
skewness in acquisition for global field reconstruction and isolate its
empirical effect.

\noindent\textbf{Counting observations and IoT sensing.}
Latent Gaussian fields drive log-Gaussian Cox counts, while the
Skew-Normal family~\cite{Azzalini1985} provides asymmetric marginals
with tractable entropy~\cite{ArellanoValle2013}. Wireless examples
include Bernoulli energy-detection decisions~\cite{Haykin2005},
random-access counts~\cite{Baccelli2009}, NB-IoT preamble
detections~\cite{Chougrani2022} and cooperative IIoT
sensing~\cite{XLiu2023}, with UAVs serving as data
collectors~\cite{Zeng2019}. Here, counts and detections serve as
spatial measurements rather than only as traffic.

\section{System Model}
\label{sec:model}

\begin{figure*}[t]
\centering
\includegraphics[width=\textwidth]{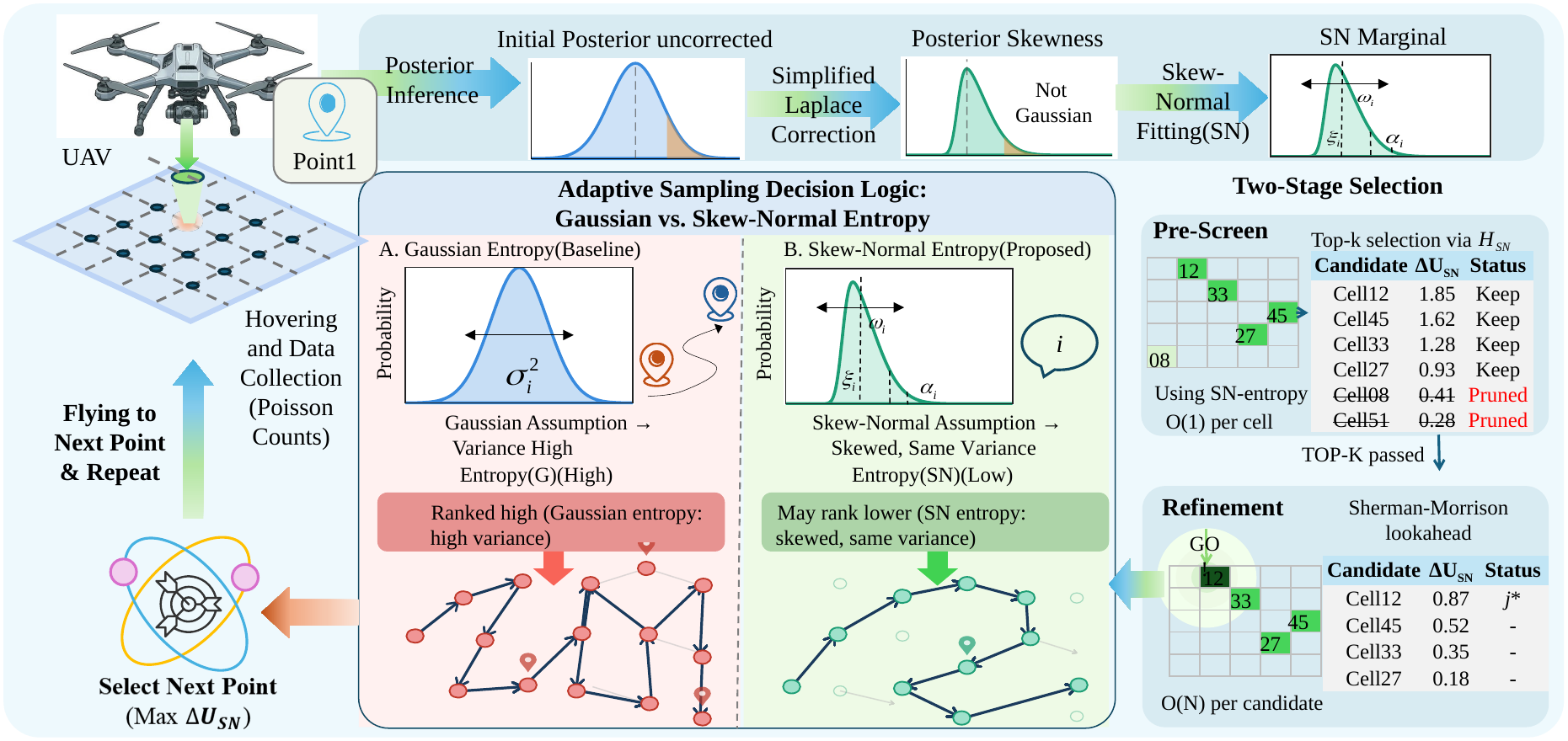}
% \caption{Adaptive sampling framework: INLA update, entropy pre-screen and
% Sherman--Morrison lookahead over the kinematic reachable set.
% \rev{Panels~A--B are schematic, no accuracy difference between the two rules
% is resolved (Section~\ref{sec:null}).}}
\caption{Adaptive UAV REM sampling: on-station observations,
ground-station posterior updates, and two-stage waypoint selection.
Skewness affects the entropy pre-screen, the fixed-shape
Sherman--Morrison lookahead is Gaussian-equivalent. Numbers and paths are schematic, the pre-screen ranks marginal SN entropy $H_{\mathrm{SN}}$.}
\label{fig:system}
\end{figure*}

The UAV sends observations to a ground station for posterior updates and
waypoint selection (Fig.~\ref{fig:system}).
\rev{The index~$t$ counts measurement cycles comprising a hop, an on-station
listening window and an update (Section~\ref{sec:kinematics}), elapsed time
is not simulated.}

\subsection{Spatial Field}
The region comprises $N$ cells at $\{s_1,\ldots,s_N\}$. The latent field uses
natural-log units, $x_j = (P_j - P_{\mathrm{ref}})/\kappa_P$ with
$\kappa_P = 10/\ln 10 = 4.3429$\,dB per nat, so $e^{x_j}$
in~\eqref{eq:poiss} is the linear power ratio to $P_{\mathrm{ref}}$.
\rev{We use $\vx \sim \mathcal{N}(\mathbf{m}, \mQ^{-1})$, where
$\mathbf{m}$ is a known log-distance path-loss trend in the same units
and the shadowing residual has a zero-mean GMRF prior. The trend enters
\eqref{eq:poiss} as a known offset.}
The Mat\'{e}rn--SPDE representation~\cite{Lindgren2011} gives
\begin{equation}\label{eq:Q-spde}
\mQ = \tau^2\,\mathbf{K}_{\kappa^2}\mathbf{C}^{-1}\mathbf{K}_{\kappa^2},
\end{equation}
where $\mathbf{C}$ and $\mG$ are the mass and stiffness matrices,
$\mathbf{K}_{\kappa^2}=\kappa^2\mathbf{C}+\mG$,
$\kappa=\sqrt{8\nu}/\rho$ sets the range~$\rho$, and
$\tau^2=1/(4\pi\kappa^2\sigma_s^2)$ for $d{=}2$, $\nu{=}1$.
The prior shadowing standard deviation $\sigma_P$ is quoted in dB,
with $\sigma_s = \sigma_P/\kappa_P$.

\subsection{Observation Models}
At step~$t$, the UAV visits cell~$l_t$ and observes $y_t$
conditioned on $x_{l_t}$.

\textbf{Model~A: spectrum sensing (Bernoulli)}~\cite{Haykin2005},
\begin{equation}\label{eq:bern}
y_i \mid x_i \sim \mathrm{Bern}\!\bigl(\sigma(\beta_0 + \beta_1 x_i)\bigr),\quad
\sigma(\eta) = 1/(1+e^{-\eta}),
\end{equation}
has bounded, state-dependent curvature
$c_j = \beta_1^2 p_j(1-p_j)$, where
$p_j=\sigma(\beta_0+\beta_1\mu_j)$.

\textbf{Model~B: activity counting (Poisson)}~\cite{Baccelli2009},
\begin{equation}\label{eq:poiss}
y_i \mid x_i \sim \mathrm{Poisson}(N_e \exp(x_i)),
\end{equation}
has curvature $c_j = N_e e^{\mu_j}$, which can vary by orders of magnitude.
Unlike additive-Gaussian REM models~\cite{Polyzos2024},
\rev{these likelihoods do not impose the constant-curvature equivalence
of local criteria described in Section~\ref{sec:intro}.}

\subsection{\rev{Physical Instantiation of the Counting Model}}
\label{sec:iot-instantiation}

\rev{We instantiate Model~B as idealized LoRaWAN uplink traffic in EU863-870.
We assume devices transmit at the duty-cycle ceiling, independent attempts
superpose approximately as a Poisson process, adaptive data rate selects the
lowest spreading factor that closes the link, and every attempt within the
footprint is detected and attributed to its cell.}

In the $868.0$--$868.6$\,MHz sub-band, devices are limited to $25$\,mW
e.r.p.\ and $\le 1\%$ duty cycle~\cite[Table~B.1, band~M]{EN300220_2},
measured over one hour~\cite[cl.~5.4.1]{EN300220_1}, default-channel
restrictions are specified in~\cite[Table~2]{LoRaWAN_RP102}.
At the ceiling, $r_{\mathrm{dev}}(\mathrm{SF}) =
\delta/T_{\mathrm{air}}(\mathrm{SF})$ with $\delta = 0.01$.
Across the spreading-factor ladder, each sensitivity
step~\cite[Table~10]{SX1276} buys about $2.6$\,dB of margin while roughly
doubling time on air~\cite[\S4.1.1.7]{SX1276}.
The permitted rate is approximated by
$r_{\mathrm{dev}} \propto p^{\,b}$ with $b = 1.052$ from the endpoint
chord and $1.072$ from a six-rung least-squares fit%
\rev{ ($20$-byte payload, $125$\,kHz, coding rate $4/5$).
The exponent-one link in~\eqref{eq:poiss} differs by
$(p/p_0)^{b-1}$, within $\pm11\%$ over the $13$\,dB ladder when anchored
at its centre, with larger discrepancies outside it.}

\rev{
Each measurement aggregates attempts at the arrival cell~$l_t$
according to~\eqref{eq:poiss}, over an on-station listening window
$T_{\mathrm{obs}}$ and footprint $A_{\mathrm{meas}} = h^2 = 25$\,m$^2$.
No observations are collected in motion or shared between cells.
The exposure is
$N_e = \rho_{\mathrm{dev}} r_{\mathrm{dev}} h^2 T_{\mathrm{obs}}$,
where $\rho_{\mathrm{dev}}$ is device density.}
At the IMT-2020 connection density
$\rho_{\mathrm{dev}} = 10^{6}$\,km$^{-2}$~\cite[cl.~4.8]{ITU2410},
the reference SF10 rate is $r_{\mathrm{dev}} = 0.02698$\,s$^{-1}$,
\rev{giving
\begin{equation}\label{eq:Ne-number}
N_e = 0.6745\,T_{\mathrm{obs}}/(1\,\mathrm{s}), \qquad
N_e = 1 \,\Leftrightarrow\, T_{\mathrm{obs}} = 1.48\,\mathrm{s}.
\end{equation}
Thus, $N_e \in \{1,2,5\}$ corresponds to dwells of $1.5$ to $7.4$\,s.}
The propagation anchor ($P_0 = -24.43$\,dBm at $1$\,m,
$n = 3.76$~\cite[cl.~4.5.2]{TR36942}, gateway distance $623.0$\,m)
sets $P_{\mathrm{ref}} = -129.50$\,dBm at the centre of the
$-123$ to $-136$\,dBm ladder.
The rate represents a regulatory ceiling, not predicted traffic,
and counts represent attempts, not decoded frames.
The exponential branch covers $99.5$ to $57.8\%$ of ROI cells
as $\sigma_P$ increases from $2$ to $8$\,dB.

\subsection{\rev{UAV Kinematics and the Reachable Set}}
\label{sec:kinematics}

The rotary-wing UAV flies at constant altitude and cruise speed $v_{\max}$
and can hover and yaw in place\rev{, acceleration, deceleration and turns
are instantaneous. The travel allowance $\Delta t$ excludes sensing and
computation: each hop covers at most $d_{\max} = v_{\max}\Delta t$,
followed by an on-station listening window $T_{\mathrm{obs}}$
(Section~\ref{sec:iot-instantiation})}.
The reachable set is
\begin{equation}\label{eq:reachable}
\mathcal{A}_t = \bigl\{\, j \,:\, \|s_j - s_{l_t}\|_2 \le d_{\max} \,\bigr\}
\end{equation}
intersected with the ROI grid. Table~\ref{tab:params} gives
$d_{\max}=20$\,m and \rev{an interior $|\mathcal{A}_t| = 49$,
reduced near the boundary}.
We report measurement count and flown distance
$D = \sum_t\|s_{l_{t+1}}-s_{l_t}\|_2$ separately.
\rev{Under fixed-speed cruise and hover~\cite{ZengEnergy2019},
propulsion energy is affine in $D$ with a slope whose sign depends on
whether mission duration or dwell is fixed, so no energy saving is
claimed.}

\begin{table}[t]
\centering
\caption{\rev{Simulation parameters.}}
\label{tab:params}
\footnotesize
\setlength{\tabcolsep}{2.8pt}
\begin{tabular}{llc}
\toprule
Symbol & Meaning & Value \\
\midrule
$N$, $h$ & \rev{ROI cells synth./ray-tr.}, spacing & \rev{$32^2$\,/\,$19^2$}, $5$\,m \\
$\rho$, $\nu$, $d$ & Mat\'{e}rn range, smoothness, dim. & $50$\,m, $1$, $2$ \\
$\kappa_P$ & dB per nat, $10/\ln 10$ & $4.3429$ \\
$\sigma_P$ & prior shadowing s.d., nominal & $6.0$\,dB \\
$\sigma_P'$ & prior shadowing s.d., data-matched & $10.025$\,dB \\
$P_{\mathrm{ref}}$ & \rev{ref.\ level, anchor/ray-traced} & \rev{$-129.50$/$-120.65$\,dBm} \\
$n$ & \rev{path-loss exp., anchor/ray-traced} & \rev{$3.76$/$3.90$} \\
$v_{\max}$, $\Delta t$ & cruise speed, \rev{travel allowance} & $10$\,m/s, $2$\,s \\
$d_{\max}$, $|\mathcal{A}_t|$ & reachable radius, \rev{interior} set size & $20$\,m, $49$ \\
$T$ & budget, synthetic / ray-traced & $200$ / $70$ \\
$N_e$ & \rev{exposure, synthetic/ray-traced} & \rev{$1$--$100$/$1$} \\
$\beta_0$, $\beta_1$ & Bernoulli threshold, sensitivity & $-3.5$, $0.35$ \\
$K$, $K_s$ & pre-screen size, refresh period & $20$, $10$ \\
$\varrho$, $\epsilon$, $k_{\max}$ & damping, tol., cap \rev{(main/later)} & \rev{$0.8/1.0$}, $10^{-6}$, \rev{$50/100$} \\
$\kappa_{\mathrm{UCB}}$ & GP-UCB exploration coefficient & $20$ \\
$R$ & repetitions, synthetic / ray-traced & $50$ / $5{\times}19$ \\
\bottomrule
\end{tabular}
\end{table}

\section{INLA-Based Posterior Approximation}
\label{sec:inla}

\rev{We use INLA with fixed hyperparameters, except in the online-estimation
experiments of Section~\ref{sec:results}, hyperparameter uncertainty is not
integrated out.}
The Gaussian approximation $\mathcal{N}(\vmu, \mQ_*^{-1})$ is obtained
by Newton--Raphson on
$f(\vx) = -\tfrac{1}{2}(\vx-\mathbf{m})^\top\mQ(\vx-\mathbf{m})
+ \sum_t \ell_t(x_{l_t})$~\cite{Rue2009}.
At iteration~$k$,
\begin{equation}\label{eq:ci-bi}
w_i = -\!\!\sum_{t:l_t=i}\ell_t'' \ge 0, \qquad
b_i = w_i\mu_i^{(k)} + \sum_{t:l_t=i}\ell_t' + [\mQ\mathbf{m}]_i ,
\end{equation}
\rev{where $w_i$ is the accumulated likelihood curvature at cell~$i$.}
With $\mQ_* = \mQ + \mathrm{diag}(\mathbf{w})$, the Newton proposal is
$\mQ_*^{-1}\mathbf{b}$, with damping and stopping specified in
Algorithm~\ref{alg:main}. The resulting approximation has mean $\vmu$
and covariance $\mSig = \mQ_*^{-1}$, with
$\sigma_i = \sqrt{\Sigma_{ii}}$. The simplified Laplace approximation~\cite{Rue2009} adds linear and cubic
corrections to the Gaussian log-marginal in
$z_i = (x_i - \mu_i)/\sigma_i$:
\begin{equation}\label{eq:SLA}
\log\tilde\pi_{\mathrm{SLA}}(z_i)
= \mathrm{const} - \tfrac{1}{2}z_i^2 + \gamma_i^{(1)}z_i
+ \tfrac{1}{6}\gamma_i^{(3)}z_i^3,
\end{equation}
\rev{with an $O(z_i^4)$ local remainder. For a nonzero cubic coefficient,
this expression is not globally normalizable.}
The coefficients are
\begin{equation}\label{eq:gamma3}
\gamma_i^{(3)} = \sum_{j\neq i} d_j^{(3)}
\left(\frac{\Sigma_{ij}}{\sigma_i}\right)^{\!3},
\qquad d_j^{(3)} = \sum_{t:l_t=j}\ell_t''' ,
\end{equation}
and $\gamma_i^{(1)} = \tfrac{1}{2}\sum_{j\neq i}
d_j^{(3)}(\sigma_j^2-\Sigma_{ij}^2/\sigma_i^2)\,\Sigma_{ij}/\sigma_i$%
\revb{, following~\cite[Eqs.~20--21]{Rue2009}. The location shift affects
$\xi_i$, not $\alpha_i$ or the entropy ranking}.
The cubic covariance weighting drives the decay in
Corollary~\ref{cor:convergence}.

\rev{
Equation~\eqref{eq:gamma3} and its implementation exclude $j=i$,
following~\cite[Eq.~21]{Rue2009}. Including it adds
$d_i^{(3)}\sigma_i^3$, which is zero at unobserved cells.
Against exact marginals by quadrature on a $5\times5$ lattice
(eight settings), the truncated
expression recovers $0.5$--$19\%$ of the standardized third cumulant at
observed nodes, restoring the term gives $0.3$--$7\%$ error.
At unobserved nodes, the expression agrees within $4\%$.
All reported results use the truncated correction, not a bound on the
full expansion (Section~\ref{sec:limits}).}

\rev{We identify $\gamma_i^{(3)}$ with the standardized third cumulant
only to first order: the perturbative moments of~\eqref{eq:SLA},
$\E[z_i] = \gamma_i^{(1)}+\tfrac{1}{2}\gamma_i^{(3)}$, $\E[z_i^2]=1$,
$\E[z_i^3] = 3\gamma_i^{(1)}+\tfrac{5}{2}\gamma_i^{(3)}$, give third
cumulant $\gamma_i^{(3)}+O(\gamma^2)$.}
We fit an SN density~\cite{Azzalini1985},
\begin{equation}\label{eq:SN-density}
f(x) = \frac{2}{\omega}\,\phi\!\left(\frac{x-\xi}{\omega}\right)
        \Phi\!\left(\alpha\,\frac{x-\xi}{\omega}\right),
\end{equation}
\rev{matching variance and the first-order standardized third cumulant,
with location set below}.
With $\delta = \alpha/\sqrt{1+\alpha^2}$, the standardized third cumulant
is $\kappa_3(\alpha) = C\,\delta^3/(1-2\delta^2/\pi)^{3/2}$,
where \rev{$C = \sqrt{2}(4-\pi)/\pi^{3/2} = 0.2180136$ and its limiting
magnitude is $\Gamma_{\max} = C(1-2/\pi)^{-3/2} = 0.9952717$.
Setting $\kappa_3(\alpha_i) = \gamma_i^{(3)}$ and}
$r_i = \gamma_i^{(3)}/C$, $q_i = |r_i|^{2/3}$ gives
\begin{equation}\label{eq:alpha-fit}
\begin{split}
\frac{\delta_i^3}{(1-2\delta_i^2/\pi)^{3/2}} = r_i
 \,\Longleftrightarrow\,{}&
 \delta_i = \operatorname{sign}(r_i)\sqrt{\frac{q_i}{1+2q_i/\pi}}.
\end{split}
\end{equation}
\rev{Algorithm~\ref{alg:main} uses this algebraic inverse, clipping
$|\gamma_i^{(3)}|$ below $\Gamma_{\max}$ to keep $\alpha_i$ finite.
The largest value over the nine-cell synthetic and two ray-traced
grids (later runs were not scanned), $0.429$, gives
$|\delta_i| = 0.8862$.}
Variance matching and \rev{a partial mean match} give
\begin{equation}\label{eq:omega-xi}
\omega_i = \sigma_i\big/\!\sqrt{1-2\delta_i^2/\pi},\quad
\xi_i = \mu_i + \gamma_i^{(1)}\sigma_i
- \omega_i\delta_i\sqrt{2/\pi},
\end{equation}
preserving $\Var_{\mathrm{SN}} = \sigma_i^2$ exactly.
\rev{The location omits $\tfrac{1}{2}\gamma_i^{(3)}$ from $\E[z_i]$,
up to $0.21\sigma_i$, but does not affect reported quantities:
the entropy~\eqref{eq:Hsn} depends only on $\omega_i$ and $\alpha_i$,
while NMSE and coverage use the Laplace mean and variance.}

\section{Skewness-Aware Planning}
\label{sec:plan}

\subsection{Skew-Normal Entropy and the Size of the Correction}
For $\mathrm{SN}(\xi,\omega,\alpha)$, the differential entropy
is~\cite{ArellanoValle2013}
\begin{equation}\label{eq:Hsn}
H_{\mathrm{SN}} = \tfrac{1}{2}\log(2\pi e\,\omega^2)
- \log 2 - \Psi(\alpha),
\end{equation}
where $\Psi(\alpha) =
\E_{Z\sim\mathrm{SN}(0,1,\alpha)}[\log\Phi(\alpha Z)]$.
At $\alpha=0$, $\Psi(0)=-\log 2$ recovers Gaussian entropy.
Proposition~\ref{prop:gap} quantifies the gap from a
variance-matched Gaussian.

\begin{proposition}\label{prop:gap}
Let $H_{\mathrm{G},i} = \tfrac{1}{2}\log(2\pi e\,\sigma_i^2)$ and
$\Delta H_i = H_{\mathrm{G},i} - H_{\mathrm{SN},i}$.
Then $\Delta H_i \geq 0$, with equality iff $\alpha_i = 0$, and
\revB
\begin{equation}\label{eq:sextic}
\Delta H_i = C_6\,\alpha_i^6 + O(\alpha_i^8), \ 
C_6 = \tfrac{(4-\pi)^2}{6\pi^3}
= 3.9608\!\times\!10^{-3}.
\end{equation}
\revE
\rev{The $\alpha_i^2$ and $\alpha_i^4$ coefficients both vanish.}
\end{proposition}

\begin{proof}
\rev{By~\eqref{eq:omega-xi}, $f_i$ has variance $\sigma_i^2$.
Thus $\Delta H_i$ is its Kullback--Leibler divergence from the
Gaussian of equal mean and variance, proving non-negativity and
the equality condition.}
Write $\Delta H_i = \tfrac{1}{2}\log(1-2\delta_i^2/\pi)
+ \Phi_+(\alpha_i)$, where
$\Phi_+(\alpha) = \log 2 + \Psi(\alpha)$.
Both terms are even. With
$\delta^2 = \alpha^2-\alpha^4+\alpha^6+O(\alpha^8)$ and
$u = 2\delta^2/\pi$,
\begin{equation}\label{eq:scale-exp}
\begin{split}
\tfrac{1}{2}\log(1-u) = {}& -\frac{\alpha^2}{\pi}
 + \Bigl(\frac{1}{\pi}-\frac{1}{\pi^2}\Bigr)\alpha^4 \\
 & + \Bigl(-\frac{1}{\pi}+\frac{2}{\pi^2}
 -\frac{4}{3\pi^3}\Bigr)\alpha^6 + O(\alpha^8).
\end{split}
\end{equation}
For the shape term, write
$\Psi(\alpha) =
\E_{Z\sim\mathcal{N}(0,1)}[2g(\alpha Z)]$, where
$g(u) = \Phi(u)\log\Phi(u)$.
\rev{Birnbaum's bound
$\Phi(-s) \ge 2\phi(s)/(s+\sqrt{s^2+4})$ for $s\ge0$
gives $|g^{(k)}(t)| \le P_k(|t|)\phi(t)$ for
$1\le k\le8$, with $P_k$ a polynomial having nonnegative
coefficients. For $|\alpha|\le A$, the corresponding integrand
derivatives are bounded by
$2\phi(0)\phi(z)|z|^kP_k(A|z|)$, which is integrable.
Differentiation under the integral and Taylor's theorem therefore give}
$\Phi_+(\alpha) = \sum_{m=1}^{3}
\tfrac{\E[Z^{2m}]}{(2m)!}2g^{(2m)}(0)\alpha^{2m}
\rev{{}+O(\alpha^8)}$,
with $\Phi_+(0)=0$ and odd orders zero by symmetry.
\rev{Using $\E[Z^{2m}]/(2m)! = 1/(2^m m!)$,
$g''(0) = 1/\pi$,
$g^{(4)}(0) = 4/\pi^2-4/\pi$ and
$g^{(6)}(0) = 28/\pi-80/\pi^2+96/\pi^3$ yields
\begin{equation}\label{eq:shape-exp}
\begin{split}
\Phi_+(\alpha) = {}& \frac{\alpha^2}{\pi}
 + \Bigl(\frac{1}{\pi^2}-\frac{1}{\pi}\Bigr)\alpha^4 \\
 & + \Bigl(\frac{7}{6\pi}-\frac{10}{3\pi^2}
 +\frac{4}{\pi^3}\Bigr)\alpha^6 + O(\alpha^8).
\end{split}
\end{equation}
Adding~\eqref{eq:scale-exp} and~\eqref{eq:shape-exp} cancels
the $\alpha^2$ and $\alpha^4$ terms, leaving
\begin{equation}\label{eq:c6-sum}
\frac{1}{6\pi}-\frac{4}{3\pi^2}+\frac{8}{3\pi^3}
= \frac{1}{6\pi}\Bigl(1-\frac{4}{\pi}\Bigr)^{\!2}
= \frac{(4-\pi)^2}{6\pi^3} = C_6,
\end{equation}
which proves~\eqref{eq:sextic}.}
\end{proof}

\begin{figure}[t]
\centering
\includegraphics[width=\columnwidth]{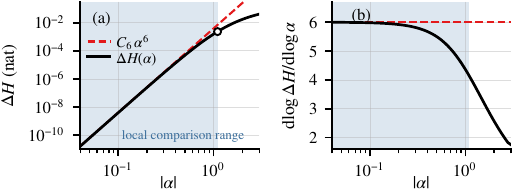}
\caption{\rev{Numerical verification of Proposition~\ref{prop:gap}.
(a) The entropy gap $\Delta H(\alpha)$ from direct quadrature of
$-f\log f$, compared with $C_6\alpha^6$.
(b) The local exponent
$\mathrm{d}\log\Delta H/\mathrm{d}\log\alpha$ tends to $6$.
Shading marks $|\alpha|\le1.1$, where $\Delta H=2.34\times10^{-3}$
nat against $7.0\times10^{-3}$ from the sextic form.}}
\label{fig:prop1}
\end{figure}

\begin{corollary}\label{cor:convergence}
\rev{Fix the prior precision~$\mQ$, the observed set $\mathcal{O}$
and a finite candidate set disjoint from $\mathcal{O}$.
Let $w_j = -\sum_{t:l_t=j}\ell_t''(\mu_j)$ be the aggregate
Poisson curvature and $w_{\min} = \min_{j\in\mathcal{O}}w_j$.
Along a sequence of Laplace posteriors with $w_{\min}\to\infty$
and candidate standard deviations
$\sigma_i \ge \sigma_{\min}>0$,
$|\gamma_i^{(3)}| = O(w_{\min}^{-2})$,
$|\alpha_i| = O(w_{\min}^{-2/3})$ and
$\Delta H_i = O(w_{\min}^{-4})$.
If the candidate Gaussian scores remain separated,
$\inf\min_{k\neq l}|H_{\mathrm{G},k}-H_{\mathrm{G},l}|=g_0>0$,
the $H_{\mathrm{SN}}$- and variance-induced pre-screen rankings
eventually coincide.}
\end{corollary}

\begin{proof}
\rev{For Poisson observations, $\ell'''=\ell''$, so
$d_j^{(3)}=-w_j$. With
$\mathbf{W}=\mathrm{diag}(\mathbf{w})\succeq\mathbf{0}$,
$\bm{\Sigma}=(\mQ+\mathbf{W})^{-1}$ obeys
$\|\bm{\Sigma}\|\le\|\mQ^{-1}\|$.
For $j\in\mathcal{O}$ and $i\notin\mathcal{O}$,
$(\mQ+\mathbf{W})\bm{\Sigma}=\mathbf{I}$ gives
$w_j\Sigma_{ij}=-[\mQ\bm{\Sigma}]_{ji}$, hence
$|\Sigma_{ij}|\le C_Q/w_j$, where
$C_Q=\|\mQ\|\,\|\mQ^{-1}\|$.
Each term of~\eqref{eq:gamma3} therefore satisfies
$|d_j^{(3)}|\,|\Sigma_{ij}|^3/\sigma_i^3
\le C_Q^3/(\sigma_{\min}^3w_j^2)
\le C_Q^3/(\sigma_{\min}^3w_{\min}^2)$.
Since $\mathcal{O}$ is fixed, the first rate follows,
\eqref{eq:alpha-fit} and Proposition~\ref{prop:gap} give the others.
The rankings coincide once $\max_i\Delta H_i<g_0$.

The assumptions matter: increasing only one $w_j$ need not
eliminate the other terms in~\eqref{eq:gamma3}, and tied Gaussian
scores can be separated by unequal entropy gaps.
The result concerns pre-screening only, not agreement of the
covariance lookahead~\eqref{eq:DeltaU-cancel} with MaxVar.}
\end{proof}

\subsection{Sherman--Morrison Lookahead and Two-Stage Selection}

The UAV scores reachable candidates by the summed marginal entropy reduction under a frozen-curvature, fixed-shape surrogate,
$\Delta U_{\mathrm{SN}}(j)$. Rather than refitting INLA for each candidate,
we use the rank-one precision update
$\mQ_*^{\mathrm{new}} = \mQ_* + c_j\mathbf{e}_j\mathbf{e}_j^\top$%
\rev{, with $c_j = -\ell''(\mu_j)$ evaluated at the current mean for
\emph{one} new observation, positive even where $w_j=0$}.
Sherman--Morrison gives
\begin{equation}\label{eq:SM}
\sigma_i^{2,\mathrm{new}} = \sigma_i^2
- \frac{c_j\,\Sigma_{ij}^2}{1 + c_j\,\sigma_j^2},
\end{equation}
requiring $O(N)$ per candidate from one covariance column.

Holding $\alpha_i$ fixed during lookahead keeps $\omega_i/\sigma_i$
constant, so the shape terms cancel and
\begin{equation}\label{eq:DeltaU-cancel}
\Delta U_{\mathrm{SN}}(j)
= \tfrac{1}{2}\sum_{i=1}^N
\log(\sigma_i^2/\sigma_i^{2,\mathrm{new}}),
\end{equation}
the Gaussian criterion\rev{, summed over the $N$ ROI cells, buffer
nodes enter the posterior but not the objective}.
We retain the top $K$ reachable candidates by $H_{\mathrm{SN}}$
(or all if $|\mathcal{A}_t|\le K$), then maximize
\eqref{eq:DeltaU-cancel} over them.
\rev{Skewness acts only through this pre-screen,
$K=|\mathcal{A}_t|$ removes its effect on the selected action.}

For $\sigma_k^2>\sigma_j^2$, the pre-screen reverses their Gaussian
ordering only if
$\Delta H(\alpha_k)-\Delta H(\alpha_j)
> \tfrac12\log(\sigma_k^2/\sigma_j^2)$.
\rev{Using the exact gap from~\eqref{eq:Hsn}, the largest scanned
shape, $|\alpha|=1.912$, gives $1.538\times10^{-2}$\,nat,
permitting reversal only for variance differences below $3.12\%$.
The local sextic approximation instead gives $0.1935$\,nat,
exceeding the exact gap's supremum of $0.1870$\,nat and therefore
cannot be used at that shape.}
For smaller shapes, $\Delta H$ decreases by approximately a factor
of $64$ per halving of $|\alpha|$
\rev{within the sextic regime} (Fig.~\ref{fig:prop1}).

\subsection{Algorithm and Cost}

Algorithm~\ref{alg:main} specifies the warm-started Newton loop.
\rev{Its scaled sup-norm criterion uses $\epsilon=10^{-6}$ and
controls the last increment, not the distance to the mode.
At $k_{\max}$, the current iterate is retained and flagged, without
retry, fallback or exclusion.
\revb{Non-finite posterior runs are instead removed with their paired
runs, exclusion and cap-hit counts are reported with each experiment.}
The covariance uses the precision assembled one iterate before acceptance,
the stopping rule controls this lag locally but not at the cap.
With a proper GMRF prior, both likelihoods give a strictly concave
log-posterior and hence a unique mode, although the stopping rule
does not certify that it has been reached.}

\rev{The main synthetic grid and MaxMI arm predate this setting and use
$\varrho=0.8$, $k_{\max}=50$ (Table~\ref{tab:params}).
The tested fixed-data estimates are numerically unchanged, but the
effect on adaptive trajectories is not bounded.}
No step reaches the cap in the main synthetic table's $900$ runs,
affected-run rates are at most $2.1\%$ in other production jobs,
except $15.0\%$ on the $64{\times}64$ arm.

\begin{algorithm}[t]
\caption{Adaptive sampling with INLA and SN pre-screen.}
\label{alg:main}
\small
\begin{algorithmic}[1]
\REQUIRE Grid $\{s_i\}$, prior $(\vmu_0,\mQ)$, initial cell $l_1$,
         budget $T_{\max}$, $K_s$, $K$, $(v_{\max},\Delta t)$,
         $(\varrho,\epsilon,k_{\max})$
\STATE $\alpha_i \leftarrow 0$, $\delta_i \leftarrow 0$,
       $\Psi_i \leftarrow -\log 2$,
       $\vmu \leftarrow \vmu_0$,
       $d_{\max} \leftarrow v_{\max}\Delta t$
\FOR{$t = 1$ \TO $T_{\max}$}
  \STATE Observe $y_t$ at $l_t$,\quad
         $\vmu^{(0)} \leftarrow \vmu$, $k \leftarrow 0$
  \REPEAT
    \STATE Compute $w_i,b_i$ at $\vmu^{(k)}$
           via~\eqref{eq:ci-bi},
           $\mQ_* \leftarrow \mQ+\mathrm{diag}(\mathbf{w})$,
           $\tilde{\vmu} \leftarrow \mQ_*^{-1}\mathbf{b}$
    \STATE $\vmu^{(k+1)} \leftarrow
           (1-\varrho)\vmu^{(k)}+\varrho\tilde{\vmu}$,\quad
           $k \leftarrow k+1$
  \UNTIL{$\|\vmu^{(k)}-\vmu^{(k-1)}\|_\infty
          < \epsilon\max(1,\|\vmu^{(k-1)}\|_\infty)$
          \OR $k=k_{\max}$}
  \STATE $\vmu \leftarrow \vmu^{(k)}$,
         \rev{retained and counted as a cap hit if $k=k_{\max}$},
         $\mSig \leftarrow \mQ_*^{-1}$
         \rev{(last assembled $\mQ_*$)},
         $\sigma_i \leftarrow \sqrt{\Sigma_{ii}}$
  \IF{$t=1$ \OR $t \bmod K_s=0$}
    \STATE Compute $\gamma_i^{(1)},\gamma_i^{(3)}$
           via~\eqref{eq:gamma3} and its accompanying expression
    \STATE Clip $|\gamma_i^{(3)}|$ below $\Gamma_{\max}$,
           $r_i \leftarrow \gamma_i^{(3)}/C$,
           update $\delta_i$ via~\eqref{eq:alpha-fit},
           $\alpha_i \leftarrow \delta_i/\sqrt{1-\delta_i^2}$,
           cache $\Psi_i$
  \ENDIF
  \STATE Compute $H_{\mathrm{SN}}(i)$ via~\eqref{eq:Hsn},
         with $\omega_i$ from~\eqref{eq:omega-xi}
  \STATE $\mathcal{A}_t \leftarrow
         \{j \text{ in ROI}:\|s_j-s_{l_t}\|_2\le d_{\max}\}$
         \quad\eqref{eq:reachable}
  \STATE Candidates $\leftarrow$ top-$\min(K,|\mathcal{A}_t|)$
         cells of $\mathcal{A}_t$ by $H_{\mathrm{SN}}$
  \STATE $l_{t+1} \leftarrow
         \argmax_{j\in\mathrm{Candidates}}\Delta U_{\mathrm{SN}}(j)$
         via~\eqref{eq:SM}, \eqref{eq:DeltaU-cancel}
\ENDFOR
\end{algorithmic}
\end{algorithm}

\begin{table}[t]
\centering
\caption{Per-step selection cost.}
\label{tab:complexity}
\footnotesize
\begin{tabular}{lcc}
\toprule
Strategy & Selection cost & Skewness \\
\midrule
Random & $O(1)$ & No \\
MaxVar & $O(|\mathcal{A}_t|)$ & No \\
MaxMI & $O(|\mathcal{A}_t|)$ & No \\
MaxEnt-G (naive) & $O(kN^3|\mathcal{A}_t|)$ & No \\
MaxEnt-G (SM) & $O(NK)$ & No \\
MaxEnt-SN (proposed) & $O(NK)+O(NM)/K_s$ & Yes \\
\bottomrule
\end{tabular}\\
\raggedright\scriptsize
$N$: grid size, $K$: pre-screen size, $M$: observed cells,
$K_s$: refresh period, $k$: Newton iterations
(median $4$ at $\varrho=1.0$, cap $100$).
All rules share an $O(kN^{3/2})$ inner solve and
\rev{dense covariance inversion requiring $O(N^3)$ arithmetic
and $O(N^2)$ memory in the current implementation.}
\end{table}

Sherman--Morrison reduces lookahead from $O(kN^3|\mathcal{A}_t|)$
to $O(NK)$, the SN rule additionally incurs an amortized
$O(NM)/K_s$ refresh cost (Table~\ref{tab:complexity}).
Whole-step runtimes, including the shared posterior update, are
about $1.3$\,s on $32{\times}32$ and $11$--$15$\,s on $64{\times}64$
(peak memory $2.1$\,GB), one single-process run per core on AMD EPYC
9655 nodes of the KAUST Ibex cluster (NumPy 2.4.6, SciPy 1.17.1, dense
factorization, no GPU).
\rev{Computation precedes the selected hop and adds to the measurement
cycle rather than changing $\mathcal{A}_t$, these runtimes are about
$0.6$ and $6$ times the $2$\,s travel allowance, and real-time
operation is not claimed.
Elapsed mission time is not simulated.
Sparse selected inversion~\cite{Rue2009} does not supply the
covariance columns required by~\eqref{eq:SM} and~\eqref{eq:gamma3},
obtaining them by up to $K+M$ sparse solves per step is not benchmarked.}

\section{Results}
\label{sec:results}

\subsection{Setup}
Synthetic experiments use the $32{\times}32$ GMRF of
Table~\ref{tab:params}, $T=200$, $R=50$ paired seeds and
$N_e \in \{1,2,5,10,20,50,100\}$, with a common starting cell.
Ray-traced validation uses $19$ RadioMapSeer
maps~\cite{Yapar2022}, released with~\cite{Levie2021},
at $T=70$ and five seeds, giving $95$ paired runs per rule.

% \rev{The $5.9$\,GHz fields~\cite[Table~I]{Yapar2022} are paired
% with the $868$\,MHz-inspired counting model of
% Section~\ref{sec:iot-instantiation}, no frequency-transfer calibration
% is performed. The experiment tests the model on shadowed urban
% spatial structure, not an $868$\,MHz deployment.
% We define the split as ids $0$--$560$ for training and
% $561$--$700$ for testing. Of the $40$ test maps sampled every third
% id from $561$ to $678$, $19$ satisfy the retention thresholds:
% informative fraction ${\ge}\,0.90$, censored fraction
% ${\le}\,0.30$ and $1$--$99$ percentile range ${\le}\,40$\,dB.
% Their ids are $585$, $588$, $591$, $597$, $603$, $606$, $609$,
% $612$, $618$, $627$, $639$, $642$, $645$, $648$, $657$, $663$,
% $669$, $672$ and $675$.
% The constants $P_0=-33.66$\,dBm, $n=3.8995$, residual
% s.d.\ $10.025$\,dB and $P_{\mathrm{ref}}=-120.65$\,dBm are fitted
% on training maps and frozen.
% Runs use $N_e=1$, the first transmitter and a $19{\times}19$
% window on a $5$\,m lattice placed away from it.
% Noise-floor cells retain their floor values in planning and scoring.}
\rev{The $5.9$\,GHz fields~\cite[Table~I]{Yapar2022} are paired
with the $868$\,MHz-inspired counting model of
Section~\ref{sec:iot-instantiation}, without frequency-transfer
calibration. The experiment therefore tests shadowed urban
spatial structure, not an $868$\,MHz deployment.
Maps $0$--$560$ form the training set and $561$--$700$ the test set.
Of $40$ test maps sampled every third id from $561$ to $678$,
$19$ meet the retention thresholds: informative fraction
${\ge}\,0.90$, censored fraction ${\le}\,0.30$, and
$1$--$99$ percentile range ${\le}\,40$\,dB.\footnote{Retained
map ids: $585$, $588$, $591$, $597$, $603$, $606$, $609$,
$612$, $618$, $627$, $639$, $642$, $645$, $648$, $657$,
$663$, $669$, $672$, and $675$.}
Model constants are fitted on training maps and frozen.\footnote{
$P_0=-33.66$\,dBm, $n=3.8995$, residual s.d.\ $10.025$\,dB,
and $P_{\mathrm{ref}}=-120.65$\,dBm.}
Runs use $N_e=1$, the first transmitter, and a $19{\times}19$
window on a $5$\,m lattice placed away from it.
Noise-floor cells retain their floor values in planning and scoring.}

\rev{We compare Random, lawnmower, MaxVar, MaxMI, GP-UCB,
MaxEnt-G and MaxEnt-SN, MaxMI is synthetic-only.
Accuracy is NMSE $= \sum_j(\mu_j-x_j)^2/\sum_j(x_j-\bar{x})^2$
over the ROI. Residual NMSE replaces the denominator by
$\sum_j(r_j-\bar{r})^2$, where $r_j=x_j-m_j$, normalizing by
variation about the known trend.}

\revb{\emph{Statistics.} Paired contrasts are relative differences
$(\mathrm{NMSE}_A-\mathrm{NMSE}_B)/\mathrm{NMSE}_B$ per paired run,
reported as medians unless stated, $95\%$ intervals are percentile
intervals from $20\,000$ bootstrap resamples of maps or footprint
clusters carrying their seeds (of seeds on the synthetic corpus), and
synthetic tests are sign-flip permutation tests ($10^4$). The MDE is
$2.80$ standard errors of the paired mean ($80\%$ power), a design
quantity rather than a significance threshold, a rounded zero MDE
reflects identical trajectories, and no equivalence test is made. Holm
adjustment applies only to the eight misspecified cells of
Section~\ref{sec:mismatch}. The nine-cell ablation is
$\sigma_P\in\{2,3,6,8\}$\,dB${}\times{}N_e\in\{1,2\}$ at $80$ seeds
plus $(6\,\mathrm{dB},5)$ at $50$. The lattice adds a $10$-cell buffer
to each ROI, $N$ counts ROI cells.}

The lawnmower uses a boustrophedon raster with stride
$\operatorname{round}(\sqrt{N_{\mathrm{ROI}}/T})$, capped by
$\lfloor d_{\max}/h\rfloor$ to respect $\mathcal{A}_t$.
On $32{\times}32$, this gives stride $2$, the better of two
tested strides.
MaxMI selects
$\argmax_{j\in\mathcal{A}_t}\tfrac12\log(1+c_j\sigma_j^2)$.
\rev{Under non-Gaussian likelihoods, this Gaussian-model expression
is a curvature-weighted surrogate for $I(x_j,y_j)$,}
it coincides with MaxVar when $c_j$ is constant.
MaxEnt-G isolates the skewness pre-screen by using the same
inference procedure, reachable set and lookahead as MaxEnt-SN,
but Gaussian pre-screening entropy.
\rev{Unless otherwise stated, GP-UCB uses
$\kappa_{\mathrm{UCB}}=20$, justified in Section~\ref{sec:ucb}.}
Fig.~\ref{fig:nmse_vs_ne} compares synthetic reconstruction accuracy,
Table~\ref{tab:asr} reports network-level metrics.

\begin{figure}[t]
\centering
\includegraphics[width=\columnwidth]{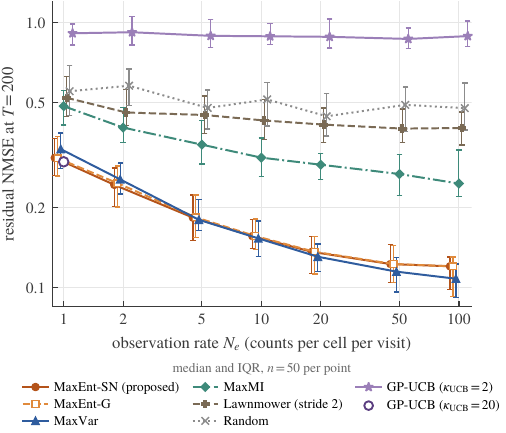}
\caption{Synthetic residual NMSE versus $N_e$ at
$\sigma_P=6$\,dB and $T=200$: medians and interquartile ranges
over $50$ paired seeds, with horizontal offsets up to $3.5\%$.
GP-UCB uses $\kappa_{\mathrm{UCB}}=2$ across the sweep,
the open marker shows $\kappa_{\mathrm{UCB}}=20$,
evaluated only at $N_e=1$.}
\label{fig:nmse_vs_ne}
\end{figure}

\subsection{Adaptive versus Non-Adaptive Sampling}
\label{sec:adaptive}

At $T=70$, posterior-informed sampling outperforms non-adaptive sampling.
\rev{Paired by map and seed, Random and the lawnmower have median
residual-NMSE penalties of $18.01\%$ and $14.01\%$ relative to
MaxEnt-SN, with higher error in $89$ and $79$ of $95$ runs.
Their $95\%$ intervals are $[+14.77,+23.66]\%$ and
$[+8.58,+20.36]\%$.
All ray-traced intervals use $20\,000$ map-cluster bootstrap
resamples\revb{, over the $19$ retained maps unless stated otherwise,}
preserving each map's five paired seeds.}
The synthetic separation is larger and emerges within the first
$30$ samples (Fig.~\ref{fig:nmse_vs_ne}).

Differences within the informative family are unresolved:
the paired relative-effect intervals for MaxVar, GP-UCB and
MaxEnt-G against MaxEnt-SN are $[-4.58,+6.25]\%$,
$[-3.41,+5.38]\%$ and $[0.00,0.00]\%$.
\rev{At higher observation rates, MaxVar overtakes MaxEnt-SN,
while MaxEnt-G remains statistically indistinguishable from it.}
Thus, the consistent separation is from non-adaptive sampling,
not between informative rules.

MaxMI adds the curvature weight $c_j$ to MaxVar's criterion.
Under Poisson counts, strongly varying $c_j$ favors high-$\mu$
cells: MaxMI \rev{has higher NMSE than} MaxVar but lower NMSE
than Random.
Under Bernoulli detections, $c_j$ varies mildly and no difference
from the other informative rules is resolved, consistent with,
but not establishing, the constant-curvature equivalence. The informative rules fly similar distances, substantially farther
than Random and the raster. The raster is cheapest in distance,
not error\rev{ (no energy ordering is inferred, Section~\ref{sec:kinematics})}.

\subsection{\rev{Prior Width Governs Calibration}}
\label{sec:calib}

Two grids of $570$ paired runs ($19$ maps, six rules, five seeds)
differ only in prior shadowing standard deviation:
$6.0$\,dB versus the training-fitted residual scale of $10.025$\,dB.
\revb{Both were fixed before testing. Section~\ref{sec:width}
extends the comparison to $4$--$16$\,dB on the same test maps
as a sensitivity analysis, not a setting-selection procedure.}

\begin{table}[t]
\centering
\caption{\rev{Prior-width intervention on the ray-traced corpus:
$570$ paired runs per width. Pooled rows aggregate all six rules, Table~\ref{tab:sweep} reports MaxEnt-SN alone. The coverage shift has a
$95\%$ map-cluster bootstrap interval of $[+4.99,+9.96]$
percentage points. Non-adaptive trajectories are identical across
widths. The exact two-sided sign-test $p$-value tests $16/19$
against one-half, not the change from $14/19$.}}
\label{tab:calib}
\footnotesize
\setlength{\tabcolsep}{2.8pt}
\begin{tabular}{lccc}
\toprule
Quantity & $\sigma_P{=}6.0$\,dB & $\sigma_P{=}10.025$\,dB & change \\
\midrule
mean $95\%$ CI coverage & $0.7585$ & $0.8314$ & \rev{$+7.29$ pts} \\
deficit against nominal & $19.15$ pts & $11.86$ pts & $-38.1\%$ \\
spread across six rules & $0.8$ pts & $0.5$ pts & --- \\
pool median NMSE, 6 rules & $0.8114$ & $0.7396$ & $-8.9\%$ \\
Random NMSE & \multicolumn{2}{c}{same flight, $882.7$\,m} & $-3.0\%$ \\
lawnmower NMSE & \multicolumn{2}{c}{same flight, $705.8$\,m} & $-8.7\%$ \\
windows beating a constant & $14/19$ & \rev{$16/19$ ($p{=}0.0044$)} & \rev{---} \\
\bottomrule
\end{tabular}
\end{table}

\rev{Widening the prior improves coverage across all six rules}%
\revb{ by $+6.65$ to $+7.69$ percentage points}%
\rev{ and reduces pooled NMSE (Table~\ref{tab:calib}).}
\revb{Coverage uses pointwise Laplace intervals for the latent field
over all ROI cells, including observed cells.}
The non-adaptive arms improve on identical trajectories,
establishing an estimator-side gain independent of adaptive sampling.
\rev{The interval score also improves: coverage rises at nominal levels
from $50$ to $99\%$, and the $95\%$ interval score~\cite{Gneiting2007},
$\mathrm{IS}_{0.05}=(u-l)+40\,[(l-x)^+ + (x-u)^+]$ in dB (width plus $40$
times the miss distance), falls by $37.6$\,dB despite a $6.9$\,dB
increase in width.}

\rev{The wider prior does not achieve calibration: coverage remains
$11.9$ percentage points below nominal at $95\%$ and below nominal
at all tested levels from $80\%$ upward ($0.88$ at $99\%$).}
Unrepresented spatial structure is a possible cause,
\rev{but covariance, likelihood and approximation errors remain
alternatives (Section~\ref{sec:limits}).}

\subsection{\revb{Choosing the Prior Width: Sweep and Online Estimation}}
\label{sec:width}

\begin{table}[t]
\centering
\caption{\revb{Prior-width sweep on retained ray-traced maps:
MaxEnt-SN, $95$ paired runs per width. NMSE is the pooled median,
coverage and interval score are means. $\Delta$NMSE is the median paired relative change and $\Delta$cov.\ the mean paired change, both against $6.0$\,dB, with map-cluster $95\%$ intervals excluding zero.
Random's penalty is the median paired relative NMSE difference
against MaxEnt-SN at each width.}}
\label{tab:sweep}
\scriptsize
\setlength{\tabcolsep}{2.4pt}
\resizebox{\columnwidth}{!}{%
\begin{tabular}{lccccccc}
\toprule
$\sigma_P$ (dB) & $4$ & $6.0$ & $8$ & $10.025$ & $12$ & $14$ & $16$ \\
\midrule
resid.\ NMSE & $0.886$ & $0.792$ & $0.761$ & $0.730$ & $0.695$ & $0.678$ & $0.673$ \\
coverage & $0.684$ & $0.756$ & $0.794$ & $0.833$ & $0.865$ & $0.885$ & $0.902$ \\
IS$_{95}$ (dB) & $152.8$ & $122.7$ & $103.6$ & $86.0$ & $74.7$ & $67.6$ & $62.0$ \\
$\Delta$NMSE (\%) & $+11.3$ & --- & $-5.0$ & $-12.0$ & $-14.4$ & $-17.3$ & $-21.0$ \\
$\Delta$cov.\ (pt) & $-7.2$ & --- & $+3.8$ & $+7.7$ & $+11.0$ & $+12.9$ & $+14.6$ \\
Random (\%) & $+9.9$ & $+13.5$ & $+17.7$ & $+18.0$ & $+18.9$ & $+21.5$ & $+23.6$ \\
\bottomrule
\end{tabular}}
\end{table}

\begin{table}[t]
\centering
\caption{\revb{Online estimation on retained ray-traced maps:
MaxEnt-SN, $95$ paired runs per row. Final width and residual NMSE
are medians, coverage is a mean. Changes are paired against fixed
$10.025$\,dB: median relative NMSE with map-cluster $95\%$ intervals,
and mean coverage change. Width estimation starts at $6.0$\,dB,
offset-only rows fix the width at $10.025$\,dB.
The true-offset row is an oracle benchmark.}}
\label{tab:eb}
\scriptsize
\setlength{\tabcolsep}{2.2pt}
\begin{tabular}{lccccc}
\toprule
Prior & $\hat\sigma_P$ (dB) & NMSE & cov. & $\Delta$NMSE (\%) & $\Delta$cov.\ (pt) \\
\midrule
fixed $6.0$\,dB & --- & $0.792$ & $0.756$ & $+13.7$ $[9.3, 16.9]$ & $-7.7$ \\
fixed $10.025$\,dB & --- & $0.730$ & $0.833$ & ref. & ref. \\
fixed $16$\,dB & --- & $0.673$ & $0.902$ & $-9.2$ $[-13.8, -6.7]$ & $+6.9$ \\
width online & $6.50$ & $0.807$ & $0.777$ & $+7.7$ $[4.4, 11.1]$ & $-5.6$ \\
offset online & $10.025$ (fixed) & $0.737$ & $0.847$ & $-1.9$ $[-4.8, +0.2]$ & $+1.4$ \\
both online & $5.36$ & $0.835$ & $0.760$ & $+10.9$ $[6.4, 17.6]$ & $-7.3$ \\
true offset & $10.025$ (fixed) & $0.711$ & $0.846$ & $-2.5$ $[-5.1, +0.0]$ & $+1.3$ \\
\bottomrule
\end{tabular}
\end{table}

\revb{\emph{A wider sweep.}
Across the seven widths in Table~\ref{tab:sweep}, residual NMSE,
coverage at every tested nominal level and interval score improve
monotonically from $4$ to $16$\,dB.
Thus, the training residual is not an optimum on this corpus:
no turning point is observed, and the sweep is a sensitivity
analysis, not a selection rule.
Section~\ref{sec:rw} shows a different curve on measured fields.
Random's penalty grows from $9.9\%$ to $23.6\%$, while
within-family differences remain unresolved:
MaxVar and GP-UCB are within $\pm3\%$ with intervals containing
zero, and MaxEnt-G has a median difference of $0.00\%$.
Both non-adaptive trajectories remain identical across widths.
The re-run $6.0$ and $10.025$\,dB settings reproduce the stored
trajectories exactly except for MaxVar, no ray-traced run is
non-finite, and cap hits are at most $4/95$ retained and $15/105$
excluded runs per cell.}

\revb{
The $21$ excluded maps contain more building area and are harder
to reconstruct: residual NMSE is $0.915$ at $10.025$\,dB,
versus $0.730$ on retained maps, falling to $0.759$ at $16$\,dB.
Coverage rises from $0.54$ at $6.0$\,dB to $0.77$ at
$10.025$\,dB and $0.89$ at $16$\,dB.
The first increase gives $+23.7$ percentage points
$[+12.5,+36.2]$ and $-21.8\%$ NMSE $[-26.3,-17.0]$.
At $10.025$\,dB, Random and raster penalties are
$20.3\%$ $[13.8,27.0]$ and $19.1\%$ $[9.9,38.1]$,
informative rules remain within $\pm1.3\%$.
The screen therefore selected easier windows without changing
the observed separation from non-adaptive sampling.}

\revb{\emph{Online width estimation.}
Every $K_s=10$ steps, we estimate $\sigma_P$ from the accumulated
counts by maximizing the Laplace marginal likelihood,
\begin{equation}\label{eq:eb}
\begin{split}
\log p(\vy\mid\sigma_P) \approx{}&
\sum_t \ell_t(\hat{x}_{l_t})
-\tfrac{1}{2}(\hat\vx-\mathbf{m})^\top
\mQ(\hat\vx-\mathbf{m})\\
&+\tfrac{1}{2}\log|\mQ|-\tfrac{1}{2}\log|\mQ_*|,
\end{split}
\end{equation}
where $\hat\vx$ is the mode under $\mQ(\sigma_P)$.
We search $3$--$16$\,dB at $1$\,dB spacing with parabolic
refinement and rescale the prior precision before the next update.
This empirical-Bayes procedure ($14$ grid evaluations per refresh)
is evaluated offline, real-time feasibility is not established.

On model-generated synthetic fields, the median width-estimation
error is $0.3$--$0.7$\,dB, residual NMSE stays within $0.04$
and coverage within $3$ percentage points of the true-prior runs.
On ray-traced fields, both $6.0$ and $10.025$\,dB starts converge
to a median near $6.5$\,dB ($[5.4,7.3]$ and $[5.4,7.4]$).
This estimate does not recover the predictive gains of wider
priors (Tables~\ref{tab:sweep}--\ref{tab:eb}), consistent with
model mismatch rather than initialization alone.}

\revb{\emph{Learning a mean offset.}
We also estimate a window-level offset $c$,
$\mathbf{m}\to\mathbf{m}+c\mathbf{1}$, by coordinate ascent
on~\eqref{eq:eb}, at fixed width or jointly with it.
True offsets span $-13.5$ to $+14.1$\,dB (median $-4.0$\,dB),
the median estimation error is $1.3$\,dB.
Even the true offset improves residual NMSE by only $2.5\%$
and coverage by $1.3$ points, against $9.2\%$ and $6.9$ points
from widening to $16$\,dB, the learned offset recovers about two
thirds of that gain, and joint estimation reduces the width to
$5.4$\,dB (Table~\ref{tab:eb}).

A window-level mean error does not explain the wider
prior's advantage. Within-window structure remains a candidate:
$3$--$29\%$ of retained cells lie at the ray tracer's noise floor,
which the stationary Gaussian residual does not explicitly model.
The tested marginal-likelihood procedure does not recover the
width associated with better prediction and coverage here,
hierarchical, predictive and robust online estimators remain untested.
Specifications and training-survey residuals provide candidate
widths, with the latter already preferred on the measured corpus
(Section~\ref{sec:rw}), but no selection rule is validated across
independent sites.}

\subsection{\rev{What Does Not Separate: Characterized Nulls}}
\label{sec:null}

\rev{Seven ablations comprise four tested contrasts with no significant
effect and three untested sensitivity probes (Table~\ref{tab:nulls}).
They reuse seeds and designs, so are not independent confirmations.}

\emph{The skewness term.}
MaxEnt-G isolates the skewness pre-screen while retaining the same
inference and lookahead.
\rev{At the two ray-traced prior widths, the correction reorders
candidates in $72$ and $80$ of $95$ runs but changes actions on only
$0.34\%$ and $0.64\%$ of steps. Identical trajectories in $78$
and $65$ runs determine the zero median accuracy differences.
Across nine $(\sigma_P,N_e)$ cells, all-run paired means lie within
$\pm1.0\%$, seed-bootstrap intervals contain zero both overall and
among divergent runs. The $80\%$-power MDE is $0.8$--$2.5\%$.
Ray-traced paired means have magnitude about $0.3\%$, with
MDE $1.0$--$1.1\%$.}

\emph{Why the effect is small.}
The sextic gap (Proposition~\ref{prop:gap}), its decay
(Corollary~\ref{cor:convergence}) and the rare action changes
support, but do not prove, the absence of a detectable benefit
in the tested regime.

\emph{Refresh period and pre-screen size.}
At $N_e=10$, varying $K_s$ from $1$ to $50$ changes NMSE by
$0.003$, and varying $K$ from $5$ to $50$ by $0.016$.
We therefore retain $K_s=10$, $K=20$.
\rev{A load switch at $t=100$ ($N_e:1\leftrightarrow20$) also leaves
the effect unresolved for $K_s\in\{1,10,50\}$, the switch itself
raises NMSE up to $8.2$ times ($32.5\%$ Newton non-convergence under
rising load) because stored counts are reinterpreted at the current
rate, a history-misspecification test rather than exposure-aware
tracking.}

\begin{table}[t]
\centering
\caption{\rev{Seven ablations: four tested contrasts with no significant
effect and three untested sensitivity probes.}}
\label{tab:nulls}
\footnotesize
\setlength{\tabcolsep}{3pt}
\begin{tabular}{llcc}
\toprule
Ablation & Counterfactual & Effect & MDE / test \\
\midrule
SN correction, ray-traced & MaxEnt-G
& \rev{$0.3\%$} & \rev{$1.0$--$1.1\%$} \\
SN correction, nine cells & MaxEnt-G
& \rev{$1.0\%$} & $0.8$--$2.5\%$ \\
refresh period $K_s{\in}[1,50]$ & $K_s{=}10$
& $0.003$ abs. & --- \\
pre-screen $K{\in}[5,50]$ & $K{=}20$
& $0.016$ abs. & --- \\
refresh period, load switch & $K_s{=}10$
& $0.0\%$ & $0.0$--$4.9\%$ \\
grid $64{\times}64$ & $32{\times}32$
& \rev{not tested} & --- \\
SN correction, Bernoulli & MaxEnt-G
& $0.00\%$ & $p=0.56$ \\
\bottomrule
\end{tabular}
\end{table}

\emph{Bernoulli observations.}
Under Model~A, the correction remains unresolved against MaxEnt-G
(median $0.00\%$, $p=0.56$).
No rule clearly improves on prior-only NMSE after $T=200$,
making this a low-information consistency check.

% \subsection{\rev{The GP-UCB Exploration Coefficient}}
% \label{sec:ucb}

% \rev{GP-UCB does not fail as a class, its exploration coefficient must be set for a
% reconstruction objective rather than an optimization one.} Fig.~\ref{fig:ucb_cliff} sweeps
% $\kappa_{\mathrm{UCB}}\in\{1,2,3,5,10,20,50\}$ over four shadowing levels and two observation
% rates: accuracy improves sharply between $\kappa_{\mathrm{UCB}}=3$ and $10$ and is flat beyond.
% We use $\kappa_{\mathrm{UCB}}=20$, inside the stable regime rather than at the literature default
% $2$, there GP-UCB is comparable to the other informative rules.
% \begin{figure}[t]
% \centering
% \includegraphics[width=0.75\columnwidth]{fig_ucb_cliff.pdf}
% \caption{\rev{GP-UCB exploration sweep: median paired NMSE change
% between adjacent $\kappa_{\mathrm{UCB}}$ settings, $20$ seeds per cell.
% Solid: $N_e=1$, dashed: $N_e=2$.
% Filled markers indicate improvement on at least $15$ seeds.
% All tested conditions show a drop at $3\to5$ and flatten beyond $10$.}}
% \label{fig:ucb_cliff}
% \end{figure}

\subsection{\rev{The GP-UCB Exploration Coefficient}}
\label{sec:ucb}

\rev{GP-UCB requires tuning for reconstruction rather than optimization.}
Across four shadowing levels and two observation rates,
the sweep $\kappa_{\mathrm{UCB}}\in\{1,2,3,5,10,20,50\}$
shows sharp accuracy gains between $3$ and $10$, followed by
a plateau (Fig.~\ref{fig:ucb_cliff}).
We use $\kappa_{\mathrm{UCB}}=20$ within this stable regime,
where GP-UCB is comparable to the other informative rules.

\begin{figure}[t]
\centering
\includegraphics[width=0.65\columnwidth]{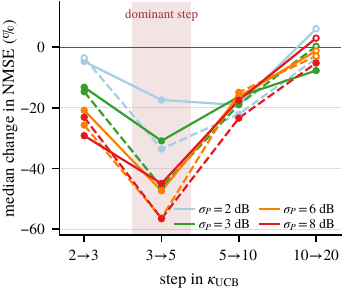}
\caption{\rev{GP-UCB sweep: median paired NMSE change between adjacent
$\kappa_{\mathrm{UCB}}$ settings, $20$ seeds per cell.
Solid: $N_e=1$, dashed: $N_e=2$.
Filled markers indicate improvement on at least $15$ seeds.
All tested conditions show a drop at $3\to5$ and flatten beyond $10$.}}
\label{fig:ucb_cliff}
\end{figure}

\subsection{Robustness to Prior Misspecification}
\label{sec:mismatch}

With synthetic truth fixed at $\rho=50$\,m and $\sigma_P=6$\,dB,
we vary the assumed range over $\{25,50,100\}$\,m and shadowing
standard deviation over $\{3,6,12\}$\,dB, using $30$ paired seeds
per cell for MaxEnt-SN and MaxVar.
Under-specifying shadowing increases residual NMSE by
$30.3$--$82.7\%$ in both rules, over-specifying it costs
$0.6$--$4.6\%$%
\rev{, except at the shortest range, where penalties reach
$29.6\%$ and $14.5\%$}.
Under-specifying range affects MaxEnt-SN more strongly:
$+14.4\%$ versus MaxVar's $-1.1\%$ ($p=0.61$).

\rev{The paired-mean ordering is not preserved.}
At the matched prior, MaxEnt-SN minus MaxVar is $-4.1\%$
relative to MaxVar
(\rev{$95\%$ interval $[-9.12,+0.85]\%$,}
permutation $p=0.063$).
It reverses to $+1.5$--$+10.8\%$ in five of nine cells:
those with under-specified range or shadowing.
\rev{Three reversals remain significant after Holm adjustment
over the eight misspecified cells, the others are unresolved
(unadjusted $p=0.054$ and $0.33$).
The mean ordering favors MaxEnt-SN only where both assumed
parameters are at or above their true values, a descriptive
finding rather than a general condition.}
Over-specification was tested at one larger value per axis
($12$\,dB, $100$\,m).
\rev{The reversal mechanism is unidentified: the two rules differ in
lookahead and pre-screen, so this grid does not isolate skewness.}

\revb{
On the retained ray-traced corpus at $10.025$\,dB, we retain
the Poisson estimator~\eqref{eq:poiss} but generate counts using
five alternatives: the ADR power law ($b=1.072$),
negative-binomial counts with Fano factors $2$ and $5$,
and detection thinning to $80\%$ and $50\%$.
Relative to matched runs, the power law and mild over-dispersion
give unresolved paired median NMSE changes of
$0.00\%$ $[-1.5,+1.2]$ and $+0.6\%$ $[-1.8,+2.6]$.
Strong over-dispersion costs $+4.2\%$ $[+0.5,+8.5]$
and $2.4$ coverage points.
Thinning produces unresolved NMSE changes
($-1.9\%$ and $-0.9\%$, intervals containing zero),
but coverage losses up to $3.1$ points.
Across all five processes, MaxEnt-SN and MaxEnt-G retain a
zero median difference, and every informative rule outperforms
Random. Thus, reconstruction accuracy tolerates the tested
link-exponent and detection-rate errors, but degrades under
strong over-dispersion.}

\subsection{\rev{A Network-Level Metric: Access Coverage at Equal Mean Power Offset}}
\label{sec:asr}

We evaluate map-based uplink power configuration.
\rev{The operator uses the posterior mean to set
\begin{equation}\label{eq:admit}
\begin{split}
a_j &= \mathrm{clip}\bigl(\Delta+\kappa_P(\bar{\mu}-\mu_j),
                         a_{\mathrm{lo}},a_{\mathrm{hi}}\bigr),\\
Q_j &= 1-\bigl(1-\exp(-10^{-(a_j+\kappa_P x_j)/10})\bigr)^L.
\end{split}
\end{equation}
where $\bar{\mu}$ is the ROI mean, $\Delta=6$\,dB,
$[a_{\mathrm{lo}},a_{\mathrm{hi}}]=[-40,23]$\,dB and $L=2$.
Offsets are relative to the power closing the link at
$P_{\mathrm{ref}}$.
The true received margin is $a_j+\kappa_P x_j$,
$Q_j$ is the probability of surviving unit-mean Rayleigh fading
within two independent attempts.
Before clipping, all maps command mean offset $\Delta$ in dB.
Access coverage $\mathrm{ACOV}$ is the ROI mean of $Q_j$.}
We report
$\eta_{\mathrm{ACOV}}=(\mathrm{ACOV}-\mathrm{ACOV}_{\mathrm{trend}})/
(\mathrm{ACOV}_{\mathrm{oracle}}-\mathrm{ACOV}_{\mathrm{trend}})$%
\rev{, computed per run and averaged, using $\mu=\mathbf{m}$
and $\mu=\vx$ as trend-only and oracle references.
Their synthetic mean access coverages are $0.8038$ and $0.9195$.}

\rev{This is a fading-averaged link-budget metric, not random-access
success: collisions, capture, scheduling and latency are excluded.
Equal mean offset in dB does not imply equal mean linear transmit
power, and clipping can change the mean offset.}

\rev{We also classify coverage holes using $z_{\mathrm{th}}$,
the $15\%$ quantile of the prior's pooled ROI marginal.
A hole has $x_j<z_{\mathrm{th}}$, the map admits a cell when
$\mu_j\ge z_{\mathrm{th}}$.
$P_{\mathrm{out}}$ counts admitted holes (a map-classification error, not physical link outage) and $P_{\mathrm{fa}}$
rejected non-holes, both as fractions of all ROI cells,
their sum is total classification error.}

\begin{table}[t]
\centering
\caption{\rev{Baseline comparison. Synthetic columns: means over
$50$ paired seeds at $\sigma_P=6$\,dB, $N_e=1$, $T=200$,
Fig.~\ref{fig:nmse_vs_ne} instead plots medians.
$\kappa=\kappa_{\mathrm{UCB}}$.
RT: median paired residual-NMSE penalty against MaxEnt-SN at
$10.025$\,dB, intervals are in Section~\ref{sec:adaptive}.}
\revb{RT $\eta_{\mathrm{ACOV}}$: normalized access coverage at
the same width, using a truth-derived window-mean reference,
averaged over the same $75$ of $95$ paired runs for every rule, those with oracle-to-trend
span above $0.02$.}
\rev{---: not evaluated.}}
\label{tab:asr}
\scriptsize
\setlength{\tabcolsep}{2.4pt}
\begin{tabular}{lcccccc}
\toprule
Acquisition rule & NMSE$_{\mathrm{resid}}$
& $\eta_{\mathrm{ACOV}}$ & $P_{\mathrm{out}}$ & $P_{\mathrm{fa}}$
& \rev{RT} & \revb{RT $\eta_{\mathrm{ACOV}}$} \\
\midrule
MaxEnt-SN
& 0.3202 & 0.6429 & 0.1072 & 0.0255
& \rev{ref.} & \revb{0.316} \\
MaxEnt-G
& 0.3215 & 0.6444 & 0.1079 & 0.0237
& \rev{$0.00\%$} & \revb{0.310} \\
GP-UCB ($\kappa{=}20$)
& 0.3208 & 0.6470 & \rev{0.1031} & \rev{0.0189}
& \rev{$+0.32\%$} & \revb{0.328} \\
MaxVar
& 0.3371 & 0.6248 & 0.1022 & 0.0199
& \rev{$+1.13\%$} & \revb{0.334} \\
lawnmower
& 0.5405 & 0.4730 & 0.1329 & 0.0121
& \rev{$+14.01\%$} & \revb{0.299} \\
Random
& 0.5631 & 0.4332 & 0.1317 & 0.0158
& \rev{$+18.01\%$} & \revb{0.215} \\
GP-UCB ($\kappa{=}2$)
& 0.9316 & 0.1660 & \rev{0.1752} & \rev{0.0003}
& \rev{---} & \revb{---} \\
\bottomrule
\end{tabular}\\
\raggedright\scriptsize
GP-UCB settings are reported separately.
\rev{MaxMI lacks stored posterior means and is omitted
(Section~\ref{sec:limits}), its median residual NMSE is $0.4841$.}
\end{table}

On the synthetic corpus, MaxEnt-SN, MaxEnt-G, MaxVar and tuned
GP-UCB recover more of the reference span than non-adaptive
sampling, but their differences remain unresolved.
\rev{MaxEnt-SN ranks first on neither decision-layer metric.}
Access coverage translates map quality\rev{ (within-field
correlation with map MSE $0.90$)} rather than validating it independently.

\revb{For ray-traced fields, we set $P_{\mathrm{ref}}$ to each
window's true mean received power, giving oracle margin $\Delta$
before clipping, this oracle reference makes $\eta_{\mathrm{ACOV}}$
a normalized diagnostic, not a deployment success rate.
At $10.025$\,dB, mean oracle, trend-only and MaxEnt-SN access
coverages are $0.9095$, $0.7785$ and $0.8098$.
MaxEnt-SN exceeds Random by $0.0114$ $[0.0073,0.0161]$
and the raster by $0.0055$ $[0.0011,0.0107]$,
winning in $83$ and $66$ of $95$ paired runs.
Informative rules differ by at most $0.004$.
The recovered fraction rises from $0.25$ at $4$\,dB to $0.33$
at $16$\,dB, Table~\ref{tab:asr} reports it only where the
oracle-to-trend span exceeds $0.02$.
Hole classification is uninformative here (almost all cells
admitted, $P_{\mathrm{fa}}<0.003$) and is not reported.}

\subsection{\revb{Measured Fields}}
\label{sec:rw}

\begin{figure}[t]
\centering
\includegraphics[width=\columnwidth]{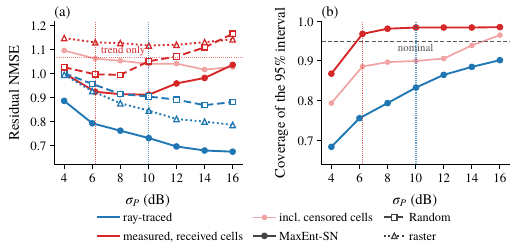}
\caption{\revb{Prior-width sensitivity: (a) pooled median residual NMSE,
(b) mean nominal $95\%$ Laplace coverage for MaxEnt-SN
(other rules differ by at most two percentage points).
Ray-traced: $19$ maps, $95$ runs per width.
Measured: $30$ Antwerp windows, $150$ runs per width,
scored on received cells, light curves additionally score
censored cells at the floor.
Vertical dotted lines mark training residual scales
($10.025$ and $6.22$\,dB), the horizontal line marks
measured-field trend-only NMSE.}}
\label{fig:rw}
\end{figure}

\revb{The Antwerp LoRaWAN corpus~\cite{Aernouts2018} (Zenodo record
3904158, v1.3) contains $130\,430$ uplinks at $868$\,MHz with GPS
positions and per-gateway SNR-corrected received powers (ESP).
We use each gateway's uplink field as a reciprocity-based proxy
for its downlink field, assuming device-power, antenna and
calibration effects are absorbed by a gateway-specific intercept
and neglecting within-band uplink--downlink differences.
Each gateway replaces the transmitter in
Section~\ref{sec:model}, defining $19{\times}19$ windows
of $30$\,m cells.

Cell truth is the mean ESP in dBm of the gateway's receptions in
the cell, treated as fixed. Cells with at least three
logged messages but no reception at the gateway are
left-censored at $-146$\,dBm, $4$\,dB below the lowest ESP in the
dataset, their simulated counts use
that floor. Cells without messages are masked from planning
and scoring. Primary metrics use received cells, censored cells
remain candidates and are separately scored at the floor,
which is a surrogate, not measured truth: non-reception is inferred
from messages logged by other gateways and depends on spreading
factor, interference and collisions, so floor-inclusive results are a
censoring-sensitivity analysis. The terrestrial field is an empirical
RSSI proxy, not an altitude-dependent UAV channel.
The experiment tests conditional reconstruction on the street
network, not full-domain mapping.

Windows require at least $30\%$ measured cells, at most
$30\%$ censoring among them, dynamic range within $40$\,dB
and an external gateway.
The test set contains $30$ windows from $12$ gateways,
with $34$--$61\%$ measured cells.
The trend is fitted on $74\,199$ messages outside these
windows: gateway-specific intercepts, common exponent $2.20$
and residual scale $6.22$\,dB.
Training and test data share gateways, so this is within-corpus
spatial generalization, not a gateway- or site-disjoint test.

Counts remain simulated by~\eqref{eq:poiss} at a fixed $N_e=1$, not
rescaled by~\eqref{eq:Ne-number} to the $900$\,m$^2$ cells (a
$0.041$\,s dwell at the reference density): the measured fields
replace the propagation model, not the LoRaWAN footprint.
We use $d_{\max}=120$\,m (four cells, a $12$\,s travel allowance
at $v_{\max}$), $|\mathcal{A}_t|\le49$,
$\rho=50$\,m, $T=70$, $N_e=1$ and five seeds per window.
Intervals resample $12$ footprint clusters and are exploratory.
Following Section~\ref{sec:plan}, $72/6300$ sweep runs with
capped Newton steps are retained.
Non-finite runs ($1/6300$ in the sweep, $2/2700$ at $N_e=5$, all
GP-UCB on one window) are removed with their pairs.}

\revb{
At $6.22$\,dB, Random and raster residual-NMSE penalties
relative to MaxEnt-SN are $+13.1\%$ $[+8.0,+32.8]$
and $+33.0\%$ $[+20.6,+56.6]$.
The separation persists across widths
($+16.5\%$ and $+33.1\%$ at $10.025$\,dB)
and grows at $N_e=5$ ($+29.3\%$ and $+52.3\%$).
Absolute residual NMSE is $0.92$ at $N_e=1$,
versus trend-only $1.07$, and $0.72$--$0.83$ at $N_e=5$.
Per-message noise is $5.1$\,dB, with a median of
$2$--$12$ messages averaged per cell.

Unlike the ray-traced result, increasing width beyond the
training residual gives no resolved NMSE gain:
the change from $6.22$ to $10.025$\,dB is
$-2.9\%$ $[-7.8,+6.6]$, while $16$\,dB incurs
$+8.8\%$ $[+1.9,+24.0]$.
Coverage is $0.968$ at $6.22$\,dB and about $0.98$
from $8$\,dB upward, interval score is lowest at $6.22$\,dB.
Thus, the training-derived width is supported here rather
than lying on a monotone improvement curve.

GP-UCB has nominally lower error by $3$--$5\%$:
$-3.6\%$ $[-8.5,+0.1]$ at $6.22$\,dB and
$-3.8\%$ $[-8.8,-1.1]$ at $8$\,dB.
MaxVar differs by at most about $3\%$, with intervals
containing zero.
Both leads grow to about $9\%$ at $N_e=5$ and reverse
at $\rho=150$\,m (MaxVar $+10.3\%$ $[+4.7,+14.7]$).
MaxEnt-G and MaxEnt-SN have identical trajectories throughout.

Censored cells constitute $8\%$ of measured cells across $23$ windows
but account for a median $52\%$ of squared error when scored
at the floor. At the training-derived width, the posterior
remains $28$\,dB above the floor, with no interval reaching it.
Including these cells reduces Random and raster penalties to
$+4.1\%$ and $+12.9\%$, NMSE continues improving to
$14$\,dB, coverage reaches $0.965$ at $16$\,dB,
and MaxVar and GP-UCB lead by $2$--$5\%$.
The preferred width therefore depends on the scoring domain:
wider intervals accommodate the floor but do not model censoring.}

\section{\rev{Discussion and Limitations}}
\label{sec:limits}

\noindent\textbf{Residual miscalibration.}
\rev{On the ray-traced corpus, nominal $95\%$ coverage remains
$0.83$ at the training-derived width}%
\revb{ and $0.90$ at the widest tested value.
The test-map sweep is a sensitivity analysis, not a width-selection
procedure}%
\rev{ (Section~\ref{sec:calib}).
Improvement on identical non-adaptive trajectories establishes
an estimator-side contribution to the deficit, but not its full cause.}
\revb{The mean-offset and likelihood interventions do not isolate
the remaining error, spatial structure, covariance mismatch and
inference approximation remain candidates.}
All arms use Gaussian intervals, the largest ray-traced third
cumulant, $0.2877$, bounds neither bias nor error tails.

\noindent\textbf{Failure cases and selection.}
Three retained windows ($603$, $663$, $618$) fail to improve
on a constant predictor, widening worsens window~$663$%
\rev{ for reasons not identified by the logs}.
Retained maps contain less building area than the full test
population ($14.3\%$ versus $23.0\%$)%
\revb{. Results on the $21$ excluded maps
(Section~\ref{sec:width}) preserve the observed separation
from non-adaptive sampling but show larger errors}.

\noindent\textbf{Scope of the acquisition comparison.}
The skewness result is limited to the tested regime
($\max|\delta|=0.886$ over the scanned runs)%
\rev{ and does not establish equivalence}.
Within-family rankings change with sampling density,
reachability and prior misspecification%
\rev{, the reversal mechanisms remain unresolved}.
\revb{GP-UCB has nominally lower error by $3$--$5\%$
on measured fields, although not every contrast is resolved
(Section~\ref{sec:rw}).}

\noindent\textbf{\rev{Truncated skewness correction.}}
\rev{Equation~\eqref{eq:gamma3} omits
$d_i^{(3)}\sigma_i^3$ at observed cells.
Restoring it at $\sigma_P\in\{6,8\}$\,dB,
$N_e\in\{1,2\}$ ($80$ seeds each) changes residual NMSE
by $-0.46$ to $+0.74\%$, with all intervals containing zero.
Its effect elsewhere is untested.}

\noindent\textbf{\revb{Field and observation validity.}}
\revb{The measured-field study is conditional on received cells,
shares gateways between training and testing, and has $12$
resampling clusters, it is not independent-site validation.
Censored cells change the width sensitivity and carry about half
the squared error at the floor, a censoring-aware model is not
implemented. All experiments use simulated counts. The five misspecification tests
(Section~\ref{sec:mismatch}) leave reconstruction accuracy
largely unchanged under the tested ADR and detection-rate
errors, but strong over-dispersion degrades it.
The observation operating point also matters: fivefold exposure
reduces measured-field NMSE by $13$--$21\%$ (Section~\ref{sec:rw}).}
\rev{The IoT traffic model remains an idealized design envelope,
with the assumptions in Section~\ref{sec:iot-instantiation}.}

\noindent\textbf{\rev{Baseline and network-metric coverage.}}
\rev{MaxMI is synthetic-only and lacks the stored posterior
means needed for Table~\ref{tab:asr}.}
\revb{Hole classification is synthetic-only and access coverage is
not evaluated on measured fields, its ray-traced normalization uses a
truth-derived reference (Section~\ref{sec:asr}) and is a link-budget
diagnostic, not measured access success.}

\noindent\textbf{Platform and temporal scope.}
The study is single-UAV, single-transmitter and lattice-based,
wind, obstacles, realistic acceleration, endurance and fixed-wing
constraints are outside~\eqref{eq:reachable}.
Hyperparameters are fixed before flight%
\revb{ except in the online-estimation arms}.
\rev{The load-switch test (Section~\ref{sec:null}) reinterprets
history at the current exposure and does not test exposure-aware
tracking.}

\section{Conclusion}
\label{sec:concl}

We developed an INLA-based framework for adaptive UAV REM sampling
from non-Gaussian IoT observations.
Posterior-informed sampling improved reconstruction over Random
and lawnmower sampling\rev{ at a fixed budget of $70$ measurements,
without an energy claim}, the Skew-Normal correction is characterized
but its isolated benefit is unresolved\revb{, and the proposed rule does
not demonstrably outperform tuned Gaussian baselines}.
Prior width materially affects accuracy and calibration\revb{, with a
preferred value that depends on the corpus and scoring domain}, widening
alone does not eliminate ray-traced undercoverage\revb{, and the tested
online estimator does not recover the wider prior, so these results are
a diagnosis, not a validated width-selection rule. Both corpora pair
propagation fields with simulated counts}.
The findings motivate improving prior\revb{ structure and observation
modelling} rather than further refining acquisition rules.

\bibliographystyle{IEEEtran}
\bibliography{IEEEabrv,references_v26}

\end{document}